\documentclass{article}

     \usepackage[preprint,nonatbib]{neurips_2020}

\usepackage[utf8]{inputenc} 
\usepackage[T1]{fontenc}    
\usepackage{hyperref}       
\usepackage{url}            
\usepackage{booktabs}       
\usepackage{amsfonts}       
\usepackage{nicefrac}       
\usepackage{microtype}      

\newcommand{\ra}[1]{\renewcommand{\arraystretch}{#1}}

\usepackage{amsthm}

\newtheorem{theorem}{Theorem}

\usepackage{graphicx}
\usepackage{comment}
\usepackage{subcaption}
\usepackage{algorithm}
\usepackage{algpseudocode}

\title{Ramanujan Bipartite Graph Products for Efficient Block Sparse Neural Networks}

\author{%
Dharma Teja Vooturi, Girish Varma, Kishore Kothapalli \\
  Center for Security Theory and Algorithmic Research \\ 
  International Institute of Information Technology Hyderabad, India \\
  \texttt{dharmateja.vooturi@research.iiit.ac.in} \\
}

\begin{document}

\maketitle

\begin{abstract}
    
    Sparse neural networks are shown to give accurate predictions competitive to denser versions, while also minimizing the number of arithmetic operations performed. However current hardware like GPU's can only exploit structured sparsity patterns for better efficiency.  Hence the run time of a sparse neural network may not correspond to the arithmetic operations required. 
    
    In this work, we propose RBGP( Ramanujan Bipartite Graph Product) framework for generating structured multi level block sparse neural networks by using the theory of Graph products. We also propose to use products of Ramanujan graphs which gives the best connectivity for a given level of sparsity. This essentially ensures that the i.) the networks has the structured block sparsity for which runtime efficient algorithms exists ii.) the model gives high prediction accuracy, due to the better expressive power derived from the connectivity of the graph iii.) the graph data structure has a succinct representation that can be stored efficiently in memory. We use our framework to design a specific connectivity pattern called RBGP4 which makes efficient use of the memory hierarchy available on GPU. We benchmark our approach by experimenting on image classification  task over CIFAR dataset using VGG19 and WideResnet-40-4 networks and achieve 5-9x  and 2-5x runtime gains over unstructured and block sparsity patterns respectively, while achieving the same level of accuracy.
    
\end{abstract}

\section{Introduction}

\begin{figure}
    \centering
    \includegraphics[width=0.9\textwidth]{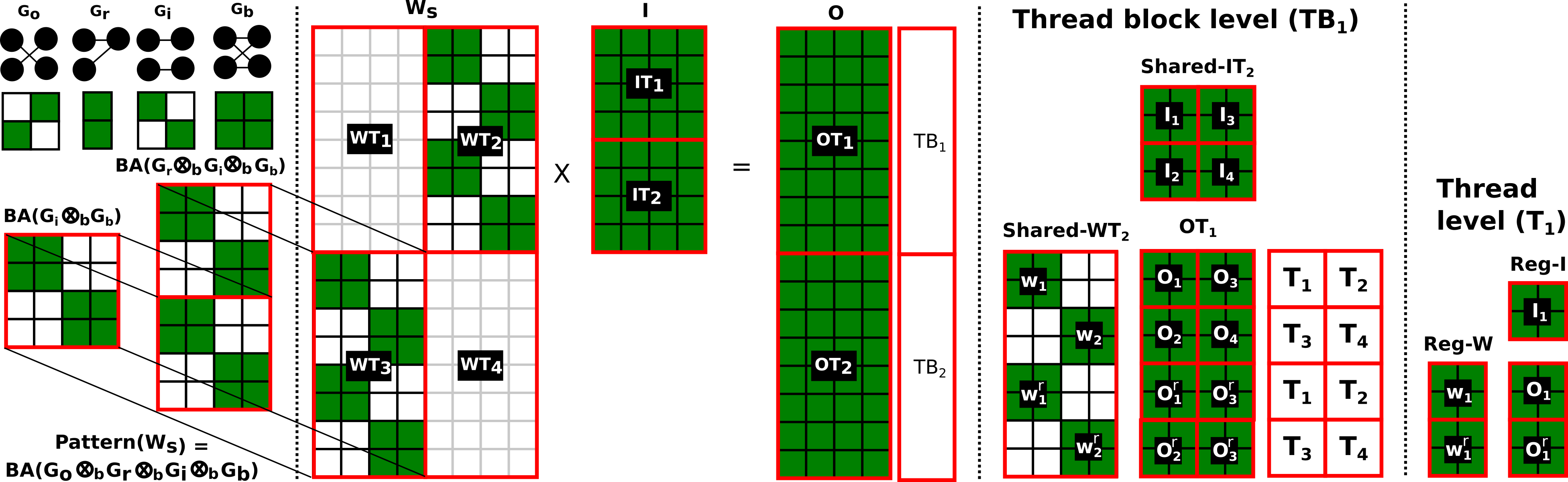}
     \caption{Tiled matrix multiplication of RBGP4 sparse matrix $W_s$ with a dense matrix $I$($O=W_s\times I)$ on GPU. A tile in $O$ ($OT$) is mapped to a thread block $TB$, and each thread in $TB$ is mapped to a 2D strided grid of element blocks in $OT$, where the number of strides, and the size of the element block in row dimension are set to $|G_r.U|$ and $|G_b.U|$ respectively. $OT$ is computed in steps, where in each step, tiles $WT$ and $IT$ are first loaded into shared memory from DRAM, and a thread in $TB$ loads corresponding elements from shared memory to registers before performing the computation. 
     }
\label{fig:rcubsmm_gpu}
\end{figure}

\setlength{\textfloatsep}{7pt}

Sparsity is an essential tool for generating compute and memory efficient neural networks. Despite this, the predominant choice of deep neural networks in production are dense instead of sparse. This is mainly because sparse neural networks tend to have poor runtime performance on the widely used dense AI hardware like GPU/TPU, that are primarily designed for accelerating dense neural networks. So in order to truly uncover the potential of sparsity in production, it is necessary to generate sparse neural networks, that are in harmony with the dense AI hardware.

Pruning \cite{lecun1990optimal,hassibi1994optimal,han2015learning,handeepcompression} is one of the widely used approach for generating sparse neural networks. In element pruning, individual parameters/elements are removed from a pre-trained dense neural network based on some criterion such as magnitude, and then the resultant sparse network is finetuned to recover accuracy. Significant number of parameters can be removed by using element pruning with minimal loss in model accuracy. But the main issue with element pruning is that the generated sparse neural networks have irregular compute and memory access patterns due to unstructured sparsity pattern, and thus cannot be efficiently mapped onto dense AI hardware. Structured pruning methods \cite{li2016pruning,molchanov2016pruning,he2017channel,luo2017thinet,Mao_2017_CVPR_Workshops,Yu_2018_CVPR,cao2019efficient,hipc19} are proposed to improve the runtime performance of sparse neural networks. Unlike element pruning, where parameters are removed at an individual level, in structured pruning, parameters are first divided into structural units like filter, channel, block, multi-block etc and then are removed at a unit level based on the strength of the unit. Structured sparse neural networks have better run-time performance than unstructured sparse neural networks. But this improvement in run-time performance comes at the cost of accuracy due to the imposed structural constraints while removing parameters from a trained model. For example, Mao et al. \cite{Mao_2017_CVPR_Workshops} have shown that for a given amount of pruning, model accuracy decreases and run-time performance increases with increase in coarsity of structural unit from 0D to 3D in pruning 4D weight tensors in convolutional neural networks. This trade-off between run-time and accuracy limits the possibility of generating efficient structured sparse neural networks using structured pruning methods. Structured sparse neural networks can also be generated using structure aware training (STAT) methods \cite{SSL, narang2017block, liu2017learning, Huang2018ECCV, Vooturi_2019_ICCV, kepner2019radix}, where structure is part of the training process. Because the structure is coupled with the training process, STAT methods are better placed than structured pruning methods in generating efficient structured sparse neural networks.

Runtime of a sparse neural network on a given hardware is dependent on the efficiency with which SDMM (Multiplication of a Sparse Matrix with a Dense matrix) operation can be implemented. On a hardware like GPU with memory hierarchy (Registers > Shared memory > L2 cache > DRAM), SDMM operation will have good runtime efficiency if and only if it maximizes data accesses from faster memory through data reuse. And for a structured sparse neural network, the amount of reuse depends on the choice of the structured sparsity pattern. Additionally, the chosen pattern should be well connected to allow for good flow of information in the neural network. In this work, we address these requirements and generate structured sparse networks that are performant and connected. Following are our main contributions:

\begin{itemize}
    \item Proposed RBGP (Ramanujan Bipartite Graph Product) framework for generating structured  sparse neural networks that have multiple levels of block sparsity, good connectivity, and takes less memory for storage.
    
    \item Using RBGP framework, we proposed RBGP4 structured sparsity pattern for the GPU, a representative dense hardware, and achieve good runtime efficiency for the SDMM (Multiplication of a sparse matrix with a dense matrix) operation on GPU.
    
    \item We  demonstrate  the  utility  of  RBGP4  sparsity  pattern  on  image  classification  task over CIFAR dataset and achieve 5-9x and 2-5x runtime gains over unstructured and block sparsity patterns respectively, while achieving the same level of accuracy.
    
\end{itemize}

\section{Related work}
\textbf{Post training:} Generating sparse neural network from a trained dense model dates back to decades old work of Lecun et al. \cite{lecun1990optimal} and Hassibi \& Stork \cite{hassibi1994optimal} where they use second-derivative information to prune weights from a dense model. The idea of pruning was revived by Han et al. \cite{han2015learning,handeepcompression} by simply pruning weights based on their magnitude. To improve runtime performance on dense AI hardware, structured pruning methods \cite{li2016pruning,molchanov2016pruning,he2017channel,luo2017thinet,Mao_2017_CVPR_Workshops,Yu_2018_CVPR,cao2019efficient,hipc19} are proposed with various structured sparsity patterns like filter,channel,block and multi-block.

\textbf{During training:} Sparse neural networks are generated during the training process either by gradually removing the connections or rearranging existing set of connections \cite{srinivas2017training, narang2017exploring, bellec2017deep, mocanu2018scalable, mostafa2019parameter, SNIP, DBLP:journals/corr/abs-1907-04840}. Similarly, structured sparse networks are generated by removing elements at a structural unit level during training. Wen et al. \cite{SSL} used group Lasso regularization to induce channel and filter sparsity in CNNs. Narang et al. \cite{narang2017block} used gradual pruning along with group  Lasso regularization to induce block sparsity pattern in RNNs. In \cite{liu2017learning, Huang2018ECCV, Vooturi_2019_ICCV}, structure is induced  by assigning a learnable parameter for each structural unit and removing them gradually through regularization and pruning.

\textbf{Before training(predefined):} Sparsity can be incorporated apriori to the training process by choosing a mask(choice of connections) in each layer of the sparse neural network and keeping it fixed through out the training. Prior works in predefined approach differ in the way the mask is chosen. Prabhu et al.\cite{Prabhu_2018_ECCV} makes use of expander graphs, and generates a random mask with row uniformity pattern, where all the rows in the mask have equal number of non zeros. Sourya et al.\cite{dey2019pre} generates a random mask with both row and column uniformity. Frankle et al. \cite{frankle2018lottery} uses an unstructured mask generated by pruning a trained dense model. Kepner et al. \cite{kepner2019radix} uses the idea of radix topology to generate a mask with cyclical diagonal pattern. Blocking pattern is the key requirement for achieving runtime performance on dense AI hardware, and none of the above works incorporate block sparsity pattern. In this work, we impose impose block sparsity pattern at multiple levels using RBGP framework, and achieve good runtime performance on GPU, a representative dense AI hardware.

\section{Preliminaries}
\label{sec-preliminaries}
In this section, we setup various definitions and notations used throughout the paper. First we define various types of block sparsity patterns.

\textbf{Block Sparse (BS) matrix:} A BS matrix $W_{bs}$ is a sparse matrix, where non zero elements are structured in the form of blocks of size $(bh,bw)$.  Matrix $W_{bs}$ has  $(W_{bs}.rows/bh \times W_{bs}.columns/bw)$ number of blocks, and a block in $W_{bs}$ is either a zero block with all zeros or a non-zero block with some or all elements as non-zeros. 

\textbf{Uniform Block Sparse  (UBS) matrix:} A UBS matrix $W_{ubs}$ is a block sparse matrix with block size $(bh,bw)$, where all the row/column blocks of size $(bh,W_{ubs}.columns)$/$(W_{ubs}.rows,bw)$ have equal number of non-zero blocks of size $(bh,bw)$. 

\textbf{Cloned Block Sparse (CBS) matrix:} A CBS matrix is a block sparse matrix with block size $(bh,bw)$, where all the non zero blocks of size $(bh,bw)$ have the same non-zero pattern. 

\textbf{Cloned Uniform Block Sparse (CUBS) matrix:} A CUBS matrix is a block sparse matrix with block size $(bh,bw)$ that is both UBS and CBS matrix with block size $(bh,bw)$.

\textbf{Recursive CUBS (RCUBS) matrix:} An RCUBS matrix $W_{s}$ is a sparse matrix with $K$ levels of blocking $B_1, ... B_K$ and following recursion: $W_{s}$ is a CUBS matrix with block size $B_1$, and a non zero block of size $B_{i}$ in $W_{s}$ is again a CUBS matrix with block size $B_{i+1}$. Figure \ref{fig:big_rbgp_example} shows an example of RCUBS matrix with three levels of blocking.

We consider the Bipartite graph $G=(U,V,E)$ representation of matrices (with dimension $|U|\times |V|$). In a biregular bipartite graph, all the vertices in $U$ and $V$ have same degree $d_l$ and $d_r$ respectively. The degree also characterizes the sparsity of such graphs. The eigenvalues of a graph $G$ are the eigenvalues of its adjacency matrix and they characterize many graph properties including connectivity \cite{Chung97}. Bipartite graph with $N$ vertices have  Eigen values $\pm \lambda_1,  ... ,\pm \lambda_{N/2}$, where $\lambda_1 \geq \lambda_2 ... \geq \lambda_{N/2}$. The \emph{spectral gap} between $\lambda_1, \lambda_2$ is a measure of the connectivity properties of the graph \cite{Alon86}. \emph{Ramanujan Graphs} are the graphs with the optimal connectivity (as measured by the spectral gap) for a given level of sparsity \cite{LubotzkyPhSa}.

\textbf{Ramanujan bipartite graph:}    A Ramanujan bipartite graph is a $(d_l,d_r)$-biregular bipartite graph, where the second largest eigenvalue $\lambda_2$ is less than or equal to  $(\sqrt{d_l-1} + \sqrt{d_r-1})$.

\textbf{Bipartite Graph Product $(\otimes_{b})$}:
Bipartite graph product$(G_p = G_1 \otimes_{b} G_2) $ takes two bipartite graphs,  $G_1(U_1,V_1,E_1)$ and $G_2(U_2,V_2,E_2)$ as the input and produces a bigger bipartite graph $G_p(U_p,V_p,E_p)$,  where  $U_p =  U_1\times U_2$, $V_p = V_1 \times V_2$, and $E_p$ is constructed using cross product of edges from $G_1$ and $G_2$ i.e, $ E_p = \{((u_1,u_2),(v_1,v_2)) | ((u_1,v_1) \in E_1 \& (u_2,v_2)) \in E_2 \}$.Bipartite graph product can also be viewed from a matrix viewpoint in the following way: \\ \\
A bipartite graph $G(U,V,E)$ can be represented as a bi-adjacency matrix $BA$ of size $(|U|,|V|)$, with $BA_{uv} = 1 $ if $(u,v) \in E$, and zero otherwise.  For the bipartite graph product$(G_p = G_1 \otimes_{b} G_2)$, bi-adjacency matrix of  $G_p$ is equal to the Tensor product($\otimes$) of the bi-adjacency matrices of the input bipartite graphs $G_1$ and $G_2$ i.e, $BA_p = BA_1 \otimes BA_2$. Figure \ref{fig:bgp_example} shows an example of bipartite graph product both from the viewpoint of both graph and matrix.

\begin{figure}
    \centering
    \includegraphics[width=0.9\textwidth]{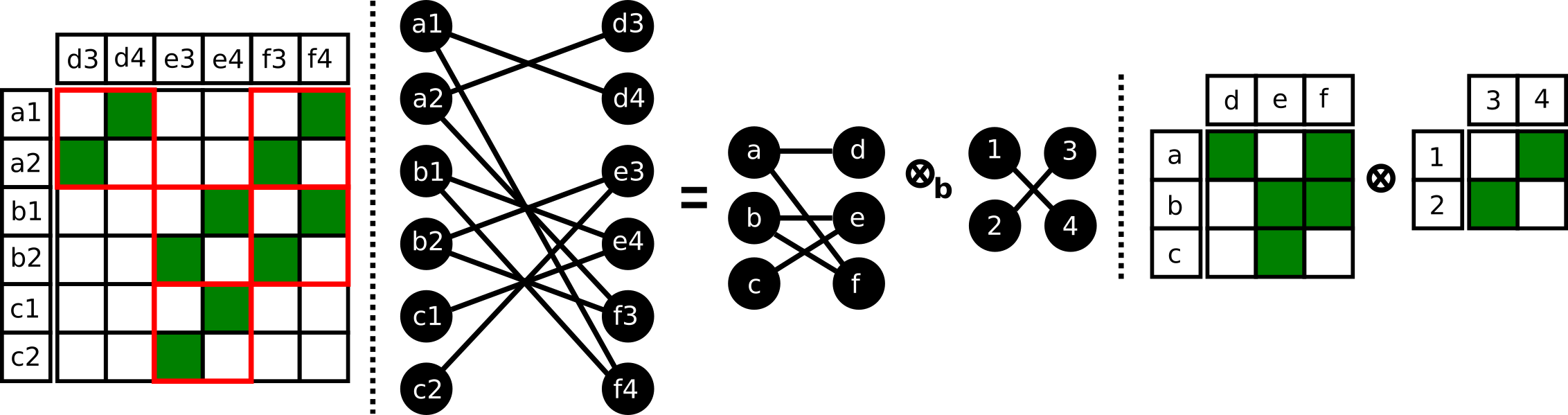}
    \caption{Bipartite graph product operation($\otimes_b$) along with matrix view. Biadjacency matrix of the product graph has CBS(Cloned Block Sparse) pattern with block size (2,2).}
    \label{fig:bgp_example}
\end{figure}

\section{Ramanujan Bipartite Graph Product Framework}
The connectivity between neurons in a layer $L$ of a sparse neural network can be captured using a bipartite graph $G$, where left/right neurons in $L$ corresponds to left/right vertices in $G$, and the connections between left and right neurons in $L$ corresponds to undirected edges between left and right vertices in $G$. The core idea in RBGP (Ramanujan Bipartite Graph Product) framework is to express $G$ as a bipartite graph product of Ramanujan bipartite graphs i.e $(G = G_{1} \otimes_b ... \otimes_b G_{K})$, where $K$ is the number of base graphs. In the rest of the section, we show how expressing connectivity of a layer using bipartite graph products leads to sparse neural networks that have structured sparsity, good connectivity, and memory efficiency.

\paragraph{Structured sparsity.}
In bipartite graph product $(G_p = G_1 \otimes_b G_2)$, the biadjacency matrix of $G_p$ is equal to the Tensor product($\otimes$) of the biadjacency matrices of $G_1$ and $G_2$ i.e, $BA_p = BA_1 \otimes BA_2$. And in Tensor product, $BA_p$ is constructed by replacing each non zero element in $BA_1$ with $BA_2$ matrix, and each zero element in $BA_1$ with zero matrix of size $BA_2$. As $BA_2$ is repeated, $BA_p$ will have CBS (Cloned Block Sparse) sparsity pattern with block size equal to the size of $BA_2$ or $(|G_2.U|,|G_2.V|)$. Figure \ref{fig:bgp_example} shows an example of bipartite graph product, where the biadjacency matrix of the product graph has CBS pattern with block size $(2,2)$. Additionally, when $G_1$ is a biregular bipartite graph, $BA_p$ will have CUBS (Cloned Uniform Block Sparse) sparsity pattern as $BA_1$ will have equal number of elements in all rows, and all columns.
In RBGP framework, the bipartite graph $G$ of a layer $L$ in the neural network is constructed by performing a series of $(K-1)$ bipartite graph products on $K$ base biregular bipartite graphs $(G = G_{1} \otimes_b \cdots \otimes_b G_{K})$ that are Ramanujan. Bipartite graph $G$ can be rewritten as $G = G_1 \otimes_b CG_2$, where $CG_2 = (G_2 \otimes_b \cdots \otimes_b G_{K})$. As $G_1$ is a biregular bipartite graph, $BA$ (biadjacency matrix of $G$) will have CUBS sparsity pattern with block size $ (\pi_{i=2}^{i=K} |G_i.U|, \pi_{i=2}^{i=K} |G_i.V|)$. Going deeper, as $CG_i = (G_i \otimes_b CG_{(i+1)})$, and also as all the base graphs are biregular, $BA$ will have RCUBS (Recursive Cloned Uniform Block Sparse) sparsity pattern with $(K-1)$ blocking levels $B_1 \cdots B_{(K-1)}$, where $B_j = (\pi_{i=j+1}^{i=K} |G_i.U|, \pi_{i=j+1}^{j=K} |G_i.V|)$. Figure \ref{fig:big_rbgp_example} shows an example bipartite graph generated using RBGP framework that uses four base graphs and has three block sizes $(16,16), (8,8),$ and $(2,2)$.

\begin{figure}[h!]
    \centering
    \includegraphics[width=0.9\textwidth]{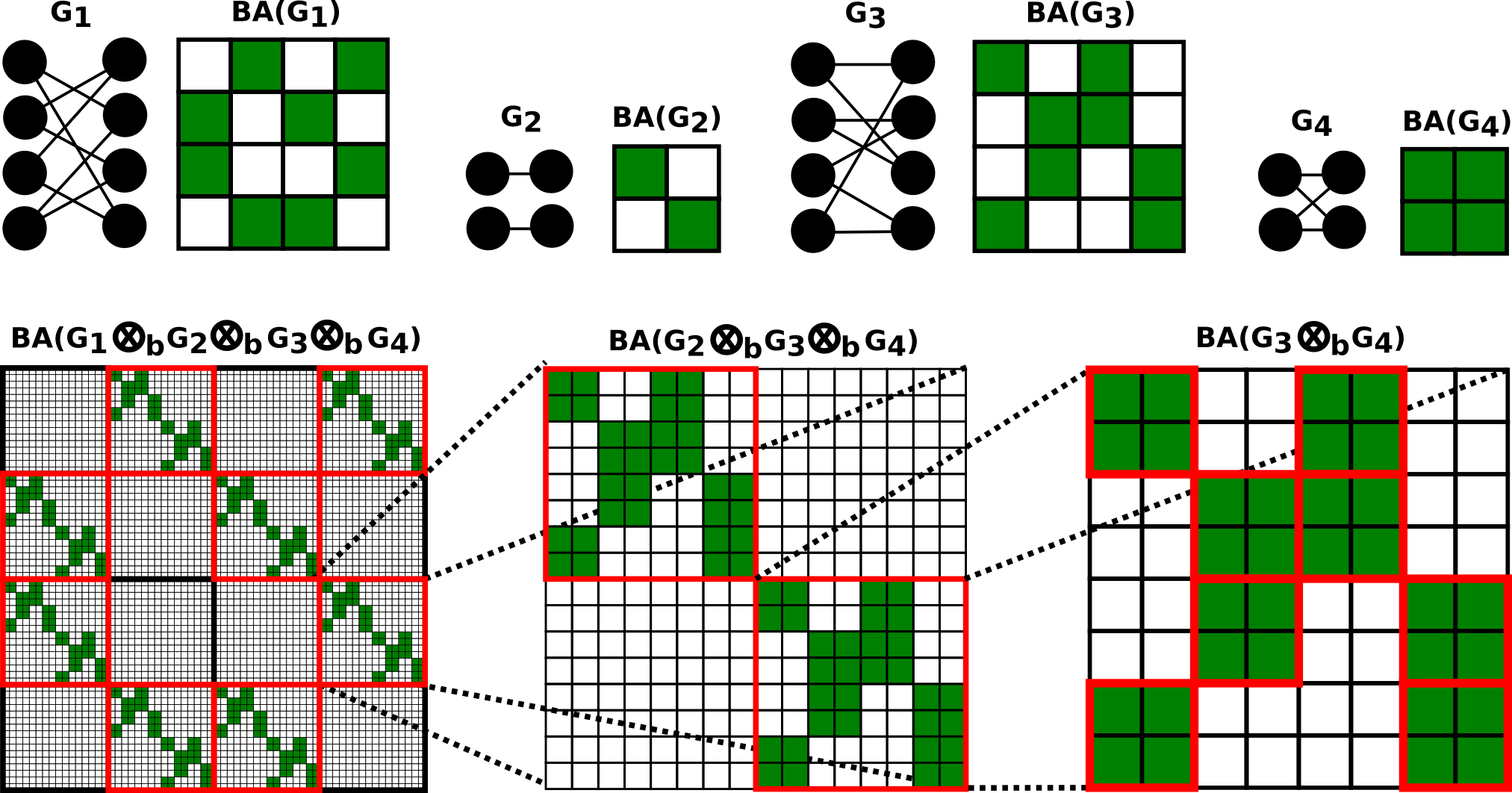}
    \caption{ Biadjacency matrix $BA$ of a bipartie graph generated using RBGP framework. $BA$ has RCUBS(Recursive Cloned Uniform Block Sparse) sparsity pattern with three blocking levels $(16,16),(8,8)$ and $(2,2)$}
    \label{fig:big_rbgp_example}
\end{figure}

\paragraph{Memory efficiency.} 
A sparse neural network can be efficiently stored by only storing the information related to the connections that are present in the sparse layers. For a sparse layer $L$ and it's associated bipartite graph $G$, $|E(G)|$ memory is required for storing the parameters corresponding to connections, and another $|E(G)|$ memory is required for storing connectivity information in the form of adjacency list of $G$. Thus a total of  $2\times |E(G)|$ memory is required for storing the information of a layer in a sparse neural network. But in a RBGP sparse neural network, the memory requirement can be reduced by reducing the memory required for storing connectivity information. In RBGP sparse neural network, as $G$ is constructed using $K$ base bipartite graphs $(G = G_{1} \otimes_b ... \otimes_b G_{K})$, the connectivity information of $G$ can be reduced from $E(G)(\prod_{i=1}^{i=K} |E(G_i)|)$ to  $\sum_{i=1}^{i=K} |E(G_i)|$, by only storing the connectivity information of the individual base graphs. For example, the bipartite graph $G$ generated using RBGP framework in Figure \ref{fig:big_rbgp_example} has 512 edges ($8\times 2\times 8\times 4$), but it only requires storing 22 edges ($8+2+8+4$) from the base graphs to construct the connectivity information of $G$, thus leading to a 23x reduction in memory requirement for  storing the connectivity information when compared to a random bipartite graph with same number of edges as $G$.

\paragraph{Good connectivity.}
Connectivity in a sparse neural network is key for ensuring good flow of information. It is well known \cite{Alon86} that connectivity of the graph is characterized by the \emph{spectral gap} between the largest and second largest eigenvalue (in absolute terms) of the adjacency matrix. In this section, we show that the spectral gap for the block sparse graph we construct using graph products, are optimal for any level of sparsity, for large graphs.

For a $d$-regular bipartite graph the largest eigenvalue in absolute value is $d$ and $-d$. The next largest eigenvalue is considered as the second largest eigenvalue $\lambda_2$. The spectral gap is $d - \lambda_2$ and larger this quantity, the better connected the graph. Suppose the bipartite graph has $n$ vertices on both sides, the degree $d$ is $\alpha n$ where $\alpha$ is the fractional sparsity. For a given value of $d$, the best possible spectral gap of $d-2\sqrt{d-1}$ is achieved by Ramanujan Graphs. We construct block sparse graphs using graph products of smaller Ramanujan Graphs and show below that this construction has similar spectral gap as $n\rightarrow \infty$. For simplicity we consider the case where the bipartite graph $G$ is the graph product of $G_1, G_2$ which are bipartite graphs with $n$ vertices on each sides and degree $d= \alpha n$. Note that $G$ has degree $d^2$ and sparsity $1-(1-\alpha)^2$.
\begin{theorem}
Let $G = G_1 \otimes_b G_2$ where $G_i$ are bipartite graphs with $n$ vertices on each sides and degree $d=\alpha n$. Then for any fixed level of sparsity $\alpha$, 
\begin{equation}\label{eqn:spec_gap}
   \frac{\mbox{IdealSpectralGap}_{d^2}}{\mbox{SpectralGap}(G)} \rightarrow 1 \quad \mbox{ as } \quad n \rightarrow \infty 
\end{equation}
where $\mbox{IdealSpectralGap}_{d^2} = d^2 - 2\sqrt{d^2-1}$ is the best possible spectral gap for $d^2$-regular graphs and $\mbox{SpectralGap}(G)$ is the spectral gap of the block sparse graph  $G$ that we construct.
\end{theorem}
\begin{proof}
The biadjacency matrix of $G$ is the tensor product of biadjacency matrices of $G_1, G_2$. Hence the eigenvalues of the biadjacency matrix is the product of eigenvalues of biadjacency matrices of $G_1, G_2$. Since $G_1, G_2$ are Ramanujan Graphs, their second largest eigenvalue is $2\sqrt{d-1}$. Hence second largest eigenvalue of $G$ is $\lambda_2(G)=d\times 2\sqrt{d-1}$. The ideal value of second largest eigenvalue for graphs of degree $d^2$ is $2\sqrt{d^2-1}$. Hence Equation \ref{eqn:spec_gap}, becomes
$$ \frac{d^2-2\sqrt{d^2-1}}{d^2-2d\sqrt{d-1}} = \frac{1-2\sqrt{1/d^2-1/d^4}}{1-2\sqrt{1/d-1/d^2}}.$$
Hence for any fixed level of sparsity $\alpha$, $n\rightarrow \infty$ (large matrices), $d \rightarrow \infty$, the LHS of Equation \ref{eqn:spec_gap} $\rightarrow 1$.
\end{proof}

\section{RBGP framework for GPU}
A GPU is fundamentally a many core architecture with thousands of cores, and have multiple memory subsytems(DRAM, L2 cache, L1 cache/shared memory,  and registers) with data access times decreasing in  that order.The reason for  having many memory subsytems is to feed data into cores at a higher rate by avoiding data accesses to slower memory say DRAM, when data is already available on faster memory say L2 cache. On GPU, a computational task can have good runtime efficiency, if it can avoid idling of cores by maximizing memory accesses from faster memories through data reuse. Sparse neural networks with unstructured sparsity pattern offers limited data reuse due to irregular memory access patterns, and thus  has poor runtime performance on GPU. The only way for sparse neural networks to achieve good runtime performance on GPU is by embracing structured sparsity patterns. In this section, using our proposed RBGP framework, we design RBGP4 structured sparsity pattern to effectively use memory subsystems on GPU by facilitating data reuse, and achieve good runtime performance for RBGP4 sparse neural networks.

\paragraph{RBGP4 sparsity pattern.}
In RBGP framework, bipartite graph $G(G = G_{1} \otimes_b ... \otimes_b G_{K})$ corresponding to a layer in the sparse neural network is configured  by the number of base graphs$(K)$, and for each base graph $G_i$, it's
type(sparse  or  complete). RBGP4 sparsity pattern corresponds to a specific configuration, where $G$ is constructed using four base Ramanujan bipartite graphs ($G=G_{o} \otimes_p G_{r} \otimes_p G_{i} \otimes_p G_{b}$), with graphs $G_{o}$ and $G_i$ being sparse, and $G_{r}$ and $G_{b}$ being complete bipartite graphs.  Figure \ref{fig:rcubsmm_gpu} shows an example of RBGP4 sparsity pattern, where $G_o$ and $G_i$ are 50\% sparse, and $G_r$ and $G_b$ are (2,1) and (2,2) complete bipartite graphs respectively.

\paragraph{GPU Implementation.}
Compute in each layer of an RBPG4 sparse neural network is composed of RBGP4MM(Multiplication of a sparse matrix $W_s$ with RBPG4 sparsity pattern, and a dense matrix $I$) operation ($O=W_s\times I$), where $W_s$, $I$, and $O$, corresponds to sparse weight matrix, batched input activations, and batched output activations respectively. We use tiling approach for  efficiently processing RBGP4MM operation. In tiling approach, matrices are divided into tiles, and $OT$(a tile in $O$) is computed in steps, where each step is comprised of matrix multiplication of $WT_s$(a sparse tile in $W$) with $IT$(a dense tile in $I$) i.e, $OT += WT_s \times IT$. For RBGP4MM, we set tile size in $W_s$ is set  to be $(|G_t.U|,|G_t.V|)$, where $G_t = (G_r \otimes_b G_i \otimes_b G_b)$.  On GPU, we associate computation of $OT$ to a thread block, and with in a thread block, each thread maps to a strided 2D grid of element blocks in $OT$, with $|G_r.U|$ number of strides and $|G_b.U|$ element block size in row dimension. We exploit the data reuse offered by  RBPG4 sparsity pattern and make efficient use of memory hierarchy on GPU, by first loading tiles $WT_s$ and $/IT$ into shared memory in each step of $OT$, and each thread loads it's share of data into registers from shared memory before performing the computation. Figure \ref{fig:rcubsmm_gpu} shows an example of using tiling approach for RBGP4MM operation on GPU. A more detailed GPU algorithm can be found in Appendix.

\paragraph{Why RBGP4 ?}
RBGP4 sparsity pattern $(G=G_{o} \otimes_p G_{r} \otimes_p G_{i} \otimes_p G_{b})$ is designed to achieve runtime efficiency for SDMM operation ($O=W_s\times I$) on GPU. Towards that, all the four base graphs $G_o$,$G_r$,$G_i$, and $G_b$ in RBGP4 sparsity pattern have a specific role to play.

The role of $G_o$ is to reduce the number of steps required to process $OT$(a tile in $O$) by inducing sparsity at the tile level in $W_s$. Performing bipartite product to the left of $G_t$ with $G_o$ i.e, $(G=G_o \otimes_b G_t)$ results in block sparsity pattern in $W_s$ with block size $(|G_t.U|,G_t.V|)$. As we set tile size in $W_s$ to be the block size, sparsity is induced at the tile/block level in $W_s$, which inturn reduces the number of steps for processing $CT$ by skipping computation corresponding to zero tiles in $W_s$. For example in Figure \ref{fig:rcubsmm_gpu}, we can see that the number of steps required to compute $OT$ is reduced from two to one, as $W_s$ has only two non zero tiles out of four tiles due to 50\% sparsity in $G_o$.

\par The role of graphs $G_r$ and $G_b$ in RBGP4 sparsity pattern is to maximize data reuse from registers in GPU threads by inducing row repetition in $WT_s$(a tile in $W_s$). In row repetition, rows are divided into groups of equal size, where all the rows in a group have non zeros at the same locations. Having row repetition pattern in $WT_s$ implies that all the rows in a group will have same memory access patterns into $IT$, and thus allows for reuse of data from $WT_s$ and $IT$. Performing bipartite graph product to the left and right of $G_i$ with complete graphs $G_r$ and $G_b$ respectively i.e, $(G_t = G_r \otimes G_i \otimes G_b)$ results in row repetition in $WT_s$ with $|G_i.U|$ groups, and $|G_r.U|\times |G_b.U|$ rows in each group. For example in Figure \ref{fig:rcubsmm_gpu}, we can see that as $G_r$ and $G_b$ are complete bipartite  graphs with $(2,1)$ and $(2,2)$ sizes, the sparsity pattern of $WT_s$, has row repetition pattern with 4 rows. In computation associated with thread $T_1$ in $O$, rows $(1,2,5,6)$ have same non zero pattern in $WT_s$, and this allows us to load two $2\times2$ blocks from $WT_s$ and one $2\times 2$ block from $IT$ into register blocks $RegW$ and $RegI$ respectively and reuse each elements from $RegW$ and $RegI$ for 2 and 4 times respectively. 

The role of $G_i$ in RBGP4 sparsity pattern is to allow $W_s$ to have any level of sparsity even when the tile size in $W_s$ is big. When the tile size in $W_s$ is relatively large when compared to the size of $W_s$, it is not possible to obtain desired level of sparsity if a non zero tile in $W_s$ is dense. For example, if a tile in $W_s$ is of size $(64,64)$, and $W_s$ is of size $(128,64)$, only by allowing tiles in  $W_s$ to be sparse, can sparsity greater than 50\% can be obtained. Bipartite graph $G_t$ corresponds to sparsity pattern of $WT_s$, and in RBGP4 sparsity pattern $G_t=(G_r\otimes_b G_i\otimes_b G_b)$. As $G_r$ and $G_b$ are dense/complete, $G_i$ has to be sparse to achieve a desired level of sparsity in $W_s$.

\begin{table}
\centering
\ra{1.3}
\begin{tabular}{@{}llrrrrcrrrr@{}}\toprule
Sparsity & Pattern  & \multicolumn{4}{c}{VGG19} & \phantom{a}& \multicolumn{4}{c}{WideResnet-40-4} \\
\cmidrule{3-6} \cmidrule{8-11} 
in \% & & CF10 & CF100 & Mem & Time && CF10 & CF100 & Mem & Time  \\ \midrule
 $00.00$ & Dense        & 93.14   & 70.64    & 77.39      & 22     && 95.01   & 77.20     & 34.10      & 40\\   
\cmidrule{1-11}
$50.00$& Unstructured    & 92.67   & 70.31    & 77.39      & 165    && 95.42   & 77.92    & 34.10       & 241  \\
& Block     & 92.45   & 70.75    & 41.12      & 94     && 95.49   & 77.52    & 18.12      & 165       \\
& RBGP4      & 92.58   & 70.48    & 38.76      & 20     && 95.34   & 78.27    & 17.13      & 32     \\
\cmidrule{1-11}
$75.00$& Unstructured    & 91.99   & 69.32    & 38.71      & 86     && 95.10   & 76.89    & 17.05      & 135       \\
& Block     & 91.93   & 68.72    & 20.57      & 48     && 94.92   & 76.50    & 9.07       & 85        \\
& RBGP4      & 91.99   & 68.34    & 19.40      & 13     && 94.72   & 76.80     & 8.57       & 20     \\
\cmidrule{1-11}
$87.50$& Unstructured    & 90.88   & 65.41    & 19.37      & 79    && 94.48   & 75.21    & 8.53       & 102       \\& Block     & 90.62   & 65.37    & 10.30       & 25     && 94.56   & 74.55    & 4.54       & 45        \\
& RBGP4      & 90.48   & 65.39    & 9.72       & 8      && 94.38   & 75.25    & 4.30       & 16      \\
\cmidrule{1-11}
$93.75$& Unstructured    & 90.01   & 62.33    & 9.70        & 50      && 93.57   & 73.09    & 4.27       & 69        \\
& Block     & 89.40    & 62.90     & 5.16       & 14     && 93.55   & 71.86    & 2.27       & 26        \\
& RBGP4      & 89.32   & 62.79    & 4.88       & 6      && 93.53   & 72.44    & 2.16       & 14    \\
\bottomrule
\end{tabular}
\vspace{1em}
\caption{Image classification on CIFAR10 (CF10) and CIFAR100 (CF100) datasets using VGG19 and WideResnet-40-4 networks. Models are trained using predefined approach with unstructured,block, and RBGP4 sparsity patterns. For block pattern, we set block size to be $(4,4)$. Memory (Mem) is given in MB, and time is given in milliseconds for one forward pass in training.}
\label{tab:classification}

\end{table}

\section{Results}
We study the effect of RBGP4 sparsity pattern on model accuracy for the task of image classification and compare with unstructured and block structured sparsity patterns. Further more, we study the effect of changing configuration of base graphs in RBGP4 sparsity pattern on runtime. We perform all our experiments on V100 GPU, where  we benchmark unstructured and block sparsity patterns using cuSparse library, and dense pattern using cuBLAS library from NVIDIA.

\paragraph{Image classification benchmark.} 
 In this benchmark, we perform the image classification task on CIFAR dataset using VGG19\cite{Simonyan15} as adapted by Liu et al. \cite{liu2018rethinking}, and WideResnet-40-4\cite{WRN} networks. To train the models, we use predefined approach, where the mask(choice of connections) is chosen apriori to the training process. As a sparse neural network has less number of parameters,we first train the dense model and guide the sparse neural network using knowledge distillation \cite{kd-hinton}. For all our experiments, we incorporate equal amount of sparsity in all layers, except for the first layer connected to input and the final classifier layer. For the optimizer, we use SGD optimizer with momentum of 0.9 and weight decay  of 1e-4.
VGG19/WideResnet-40-4 model is trained for 160/200 epochs with  batch size of 256/128. Initial learning rate is set to 0.1. For VGG19, learning rate is multiplied by 0.1 at epochs 60,120, and 160. And for WideResnet-40-4, learning rate is multiplied by 0.2 at epochs 60,120, and 160. From Table \ref{tab:classification}, we can see that RBGP4 is as accurate as unstructured and block sparsity  patterns, but takes 2x less memory and is 5-9x faster when compared to unstructured, and is 2-5x faster when compared to block sparsity pattern.

\paragraph{RBGP4 runtime characteristics.}
RBGP4 sparse matrix $W_s$ of a given size and sparsity can be obtained in multiple ways by varying the sizes of base graphs $G_o,G_p,G_i,G_b$, and sparsities of $G_o$ and $G_i$. For example, setting sparsities of $(G_o,G_i)$ to either $(0,75\%)$ or $(50\%,50\%)$ leads to $75\%$ sparsity in $W_s$, and setting sizes of base graphs to either $((8,4),(1,1),(8,4),(1,1))$ or $((8,4),(2,2),(4,2),(1,1))$ leads to $W_s$ of size $(64,16)$. In this section, we study the effect of RBGP4 configuration on runtime of SDMM operation $(O=W_s\times I)$. For  all our experiments, we set sizes of matrices $O$,$W_s$, and $I$ to be 4096x4096.

\textit{Sparsity distribution :} In RBGP4 sparsity pattern, sparsity is solely due to presence of sparse graphs $G_o$ and $G_i$, as $G_r$ and $G_b$ are dense or complete graphs. We run experiments with  75\%,87.5\%, and 93.75\% sparsity amounts distributed between $G_o$ and $G_i$, while keeping sizes of $G_o,G_r,G_i,G_b$ fixed to $(32,128),(4,1),(32,32),(1,1)$. From Table \ref{tab:sparsity_mixture}, we can see that for a given sparsity, as sparsity of $G_o$ increases,the runtime decreases. This is because sparsity in $G_o$ incorporates sparsity at the tile level, and this reduces runtime due to skipping of computation and memory loads associated with zero tiles. For dense case(0\% sparsity), we use cuBLAS library from NVIDIA.

\textit{Row repetition :} In row repetition, matrix $W_s$ can be divided into row groups of equal size, where all the rows in a row group have non zeros exactly  at the same locations. Having row repetitions allows us to effectively reuse data from $I$ as rows have same non zero pattern. $G_r$ and $G_b$ in RBGP4 introduces $|G_r.U|\times |G_b.U|$ amount of row repetition in $W_s$ . We run experiments with 1,2, and 4 repetition amounts, while keeping size of $G_t(G_r\otimes G_i \otimes G_b)$ fixed at (128,32), and sparsity  of $G_o$ at 50\%. From Table \ref{tab:row_repetition}, we can see that increasing the size of $G_r$ or $G_b$ or both leads to improved runtime performance as repetition amount increases.

\begin{table}[h!]
\parbox{.45\linewidth}{
\centering
\ra{1}
\begin{tabular}{rrrl}\toprule
         Sp(G)\% & Sp($G_o$) \% & Sp($G_i$) \% & Time(ms)  \\ \midrule
         0 & 0 & 0 & 11.2 (1x)\\ 
         \cmidrule{1-4}
         75.00 & 0.00 & 75.00 & 5.64 (2x) \\
         & 50.00 & 50.00 & 4.44 (2.5x) \\
        \cmidrule{1-4}
        87.50 & 0.00 & 87.50 & 4.31 (2.6x) \\
         & 50.00 & 75.00 & 2.74 (4.1x) \\
         & 75.00 & 50.00 & 2.29 (4.9x) \\
        \cmidrule{1-4}
        93.75 & 0.00 & 93.75 & 3.76 (3x) \\
         & 50.00 & 87.50 & 1.93 (5.8x) \\
         & 75.00 & 75.00 & 1.44 (7.8x) \\
         & 87.50 & 50.00 & 1.22 (9.2x) \\
        \bottomrule
    \end{tabular}\vspace{1em}
    \caption{Effect of varying sparsities of sparse graphs $G_o$ and $G_i$ in RBGP4 sparsity pattern on runtime.}
    \label{tab:sparsity_mixture}
}
\hfill
\parbox{.45\linewidth}{
\centering
\ra{1.48}
\begin{tabular}{cccrrr}\toprule
        \multicolumn{2}{c}{Sizes} &\phantom{a} & \multicolumn{3}{c}{Time(ms) for Sp(G)\%} \\ 
        \cmidrule{1-2} \cmidrule{4-6}
         $G_r$ & $G_b$ && 75.00 & 87.50 & 93.75  \\ 
         \midrule
        (1,1) & (1,1) && 7.07 & 3.91 & 2.45\\
        (2,1) & (1,1) && 4.89 & 3.02 & 1.97\\
        (4,1) & (1,1) && 4.47 & 2.75 & 1.92\\
        (1,1) & (2,1) && 4.85 & 3.01 & 2.03\\
        (1,1) & (4,1) && 4.47 & 2.84 & 2.02\\
        (2,1) & (2,1) && 4.41 & 2.75 & 1.98\\ \bottomrule
    \end{tabular}\vspace{1em}
    \caption{Effect of varying sizes of complete graphs $G_r$ and $G_b$ in RBGP4 sparsity  pattern on runtime.}
    \label{tab:row_repetition}
}
\end{table}
\vspace{-2em}

\section{Conclusion}
We used ideas from extremal graph theory and combinatorics to make sparse neural networks runtime efficient. Ramanujan graphs  which gives the optimal connectivity for a given level of sparsity are used to model connections in a neural network layer. Furthermore, we obtain structured block sparsity by using products of Ramanujan graphs. We prove that the product graph also has the optimal connectivity for large matrices. For the specific case of GPUs, we describe how the block sparsity can be efficiently implemented in hardware, by exploiting the memory hierarchy through data reuse. Benchmarks of this implementation is shown to give significant runtime improvements. Similar ideas could be used for generating structured sparsity patterns that results in runtime efficient implementations in other hardware as well. For the future work, generating combinatorial structured sparsity patterns like RBGP4 during the training process could lead to more accurate models as structure is induced in a gradual manner.

\pagebreak


\bibliographystyle{splncs04}
\bibliography{main}

\newpage
\section{Appendix}
\subsection{Ramanujan Bipartite Graph Generation}
A construction for Ramanujan Bipartite graph(RBG) was given by Bilu et al. \cite{BL}. The proof that this construction obtains the optimal eigenvalue gap was given by Marcus et al. \cite{MSS}. We use algorithms(graph lifts) derived from these construction to generate Ramanujan Bipartite Graphs for a given sparsity.

\textbf{2-lift operation:} A 2-lift is an operation applied on a graph $G$ to produce a bigger graph $G_L$ that is twice as big as $G$ in both vertices and edges. In the 2-lift operation, a clone graph $G^c$ is first created and the vertex set of $G_L$ is set to be the union of vertex sets of $G$ and $G^c$ i.e, $V(G_L) = V(G) \cup V(G^c) $. The edge set of $G_L$ i.e, $E(G_L)$ is then constructed in the following way:  For an edge $(u,v) \in G$,and it's corresponding clone edge $(u^c,v^c) \in G^c$, either the identity edge pair $\{(u,v),(u^c,v^c)\}$ or the crossover edge pair $\{(u,v^c),(u^c,v)\}$  is chosen at random and added to $E(G_L)$. Figure \ref{fig:2-lift-operation} shows an example of 2-lift operation. 

\begin{figure}[h!]
    \centering
    \includegraphics[width=0.9\textwidth]{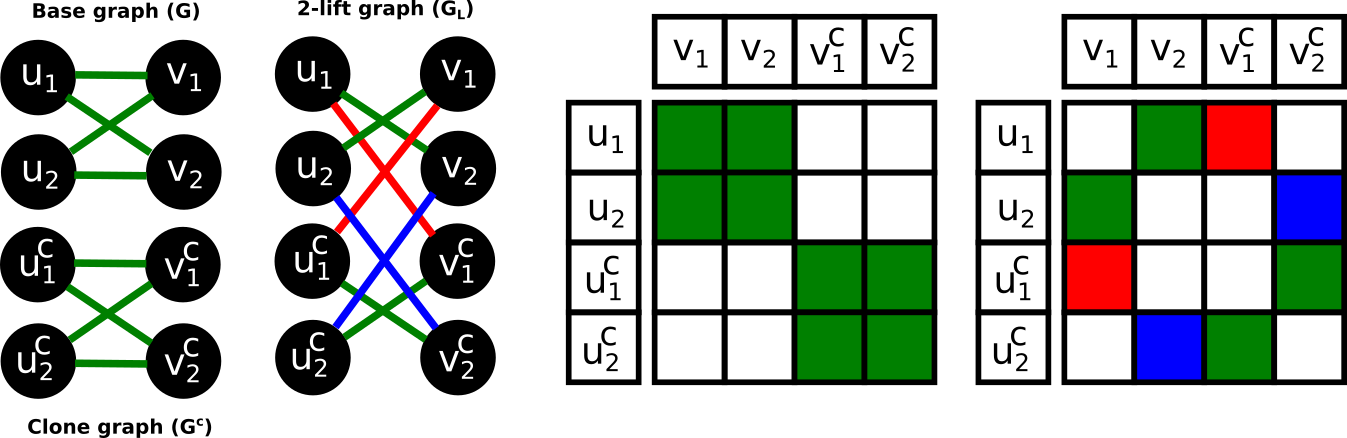}
    \caption{2-lift operation on graph $G$. Clone graph $G^c$ is first created and edges $(u_1,v_1)$ and $(u_2,v_2)$ are randomly chosen to cross over with the corresponding edges $(u_1^c,v_1^c)$ and $(u_2^c,v_2^c)$ respectively in the clone graph.}
    \label{fig:2-lift-operation}
\end{figure}

\textbf{Generating sparse biregular bipartite graph:}A 2-lift operation when applied on a biregular bipartite graph also results in a biregular bipartite graph that is twice as big with same left and right degrees. A biregular graph $G(U,V,E)$ with sparsity($1.0 - |E(G)|/(|G.U|\times |G.V|)$) $sp$, can be generated by repeatedly applying $log_{2}(1/(1-sp))$ 2-lift operations on a complete bipartite graph with $(1-sp)\times |G.U|$ left and $(1-sp)\times |G.V|$ right vertices.

\textbf{Generating RBG graph:} A Ramanujan bipartite graph is first a biregular bipartite graph with an additional constraint on second largest eigenvalue of the adjacency matrix of the graph. To generate an RBG graph, we sample sparse biregular bipartite graphs generated using 2-lift operations until the sampled graph is Ramanujan. We found that an RBG graph with sizes in the order of thousands can be generated in the order of minutes. For a layer in RBGP sparse neural network, the base Ramanujan graphs are generated only once before training and hence sampling approach is not a bottleneck.

\subsection{Pseudo code for RBGP4MM operation on GPU}
Computation in each layer of a sparse neural network is an SDMM(Multiplication of a sparse matrix with a dense matrix) operation $(C=A_s \times B)$. RBGP4MM is an SDMM operation where $A_s$ has RBGP4 sparsity pattern. Algorithm \ref{alg:rbgp4mm-gpu} describes the pseudo code for RBGP4MM operation on a GPU. As RBGP4 sparsity pattern has equal number of non zero elements in each row, non zero elements in $A_s$ can be stored using $data$ arrray of size $(A_s.rows, (1-sp)\times A_s.columns)$, and the index information of $A_s$ is captured by storing adjacency lists of base bipartite graphs.

\begin{algorithm}
\begin{algorithmic}[1]

\Function{LBFM}{$matrix, (bi, bj), (BH,BW)$} \Comment{Load Block From Matrix}
\State $block[BH][BW]$
\For{$i$ in $[0,BH)$}
\For{$j$ in $[0,BW)$}
    \State $block[i][j] = matrix[bi*BH+i][bj*BW+j]$
\EndFor
\EndFor

\State \Return $block$
\EndFunction 
\Statex
\State $G_t = G_r \otimes_b G_i \otimes_b G_b$
\State $TM,TK$ = $|G_t.U|,|G_t.V|$ \Comment{Number of left and right vertices of bipartite graph $G_t$}
\State $RM,RK$ = $|G_r.U|,|G_r.V|$ 
\State $BM,BK$ = $|G_b.U|,|G_b.V|$ 

\State $gridBlockDim = (C.rows/TM, C.cols/TN)$ \Comment{2D grid block}
\State $threadBlockDim = (TM/(RM\times BM), TN/(RN\times BN))$ \Comment{2D thread block}
\For{$(tbm,tbn)$ in $[(0,0):gridBlockDim)$} \Comment{Mapped to thread blocks}
\For{$(thm,thn)$ in $[(0,0):threadBlockDim)$} \Comment{Mapped to threads}
\State $Areg[RM][BM][BK]$ \Comment{Registers}
\State $Breg[RN][BK][BN]$ \Comment{Registers}
\State $Creg[RM][RN][BM][BN]$ \Comment{Registers}
\For{$outk$ in $[0,G_o.d_l)$} \Comment{$G_o.d_l$ is left degree of biregular bipartite graph $G_o$}
\State $oind = G_o.adj\_list[tbm][outk]$
\State $Atile = LBFM(A_s.data,(tbm,outk), (TM,G_t.d_l))$ \Comment{DRAM to shared memory}
\State $Btile = LBFM(B, (oind,tbn), (TK,TN))$ \Comment{DRAM to shared memory(shMem)}
\State $\_\_syncthreads()$
\For{$rk,ink$ in $[0,RK) \times [0,G_i.d_l)$}

\For{$rm$ in $[0:RM)$}
\State $bm$ = $rm*|G_i.U|+thm$
\State $bk$ = $rk\times G_i.d_l + ink$
\State $Areg[rm] = LBFM(Atile, (bm, bk), (BM,BK)) $ \Comment{ShMem to registers}
\EndFor

\For{$rn$ in $[0,RN)$}
\State $bk = rk\times |G_i.V| + G_i.adj\_list[thm][ink]$
\State $bn = rn\times TN/(RN\times BN) + thn$
\State $Breg[rn] = LBFM(Btile, (bk,bn), (BK,BN))$ \Comment{ShMem to registers}
\EndFor

\For{$rm,rn$ in $[0,RM) \times [0,RN)$}
\State $Creg[rm][rn] += Areg[rm] \times Breg[rn]$  \Comment{Computation}
\EndFor
\State $\_\_syncthreads()$

\EndFor
\EndFor
\For{$rm,rn$ in $[0,RM) \times [0,RN)$}
\For{$m,n$ in $[0:BM) \times [0:BN)$}
\State $row = tbm*TM + rm*(TM/RM) + thm*BM + m$
\State $col = tbn*TN + rn*(TN/RN) + thn*BN + n$
\State $C[row][col] += Creg[rm][rn][m][n] $ 
\EndFor
\EndFor

\EndFor
\EndFor
\end{algorithmic}
\caption{GPU algorithm for RBGP4MM$(C=A_s\times B)$ operation using tiling approach. Tile sizes for $A_s$,$B$, and $C$ are chosen to be $(TM,TK)$,$(TK,TN)$, and $(TM,TN)$ respectively. On GPU, each tile in $C$ is mapped to a thread block, and each thread in the thread block is mapped to a group of ($RM\times BM\times RN\times BN$) number of elements in a tile of $C$. Variables TM,TK,RM,RK,BM,BK are set based on RBGP4 configuration($G = G_o \otimes_b G_r \otimes_b G_i \otimes_b G_b)$ of $A_s$.}
\label{alg:rbgp4mm-gpu}
\end{algorithm}

\end{document}